\documentclass[twoside]{article}

\usepackage[accepted]{aistats2020}

\usepackage[round]{natbib}

\usepackage[utf8]{inputenc} 
\usepackage[T1]{fontenc}    
\usepackage[hidelinks]{hyperref}       
\usepackage{url}            
\usepackage{booktabs}       
\usepackage{amsfonts}       
\usepackage{nicefrac}       
\usepackage{microtype}      
\usepackage{algorithmicx,algpseudocode,algorithm}
\usepackage{comment}
\usepackage{enumitem}
\usepackage{thmtools, thm-restate}  

\usepackage{wrapfig,tikz}
\usepackage{amsmath,longtable,fancyhdr,multirow,graphicx,float}
\usepackage{amssymb,xcolor,amsthm}
\usepackage{subfigure}
\usepackage{color}
\usepackage{colortbl}
\usepackage{theoremref}

\newtheorem{lemma}{Lemma}

\newtheorem{definition}{Definition}

\newcommand{\be}{\begin{equation}}
\newcommand{\ee}{\end{equation}}

\definecolor{Gray}{gray}{0.85}
\definecolor{LightCyan}{rgb}{0.88,1,1}

\newcolumntype{a}{>{\columncolor{Gray}}c}
\newcolumntype{b}{>{\columncolor{white}}c}

\DeclareMathOperator*{\argsup}{arg\,sup}

\makeatletter
\usepackage{xspace}
\def\@onedot{\ifx\@let@token.\else.\null\fi\xspace}
\DeclareRobustCommand\onedot{\futurelet\@let@token\@onedot}

\newcommand\numberthis{\addtocounter{equation}{1}\tag{\theequation}}

\newcommand{\R}{\mathbb{R}}
\newcommand{\E}{\mathbb{E}}
\newcommand{\Eb}{\mathbb{E}}
\newcommand{\Nc}{{\mathcal{N}}}

\newcommand{\Zc}{\mathcal{Z}}

\newcommand{\Yc}{{\mathcal{Y}}}

\begin{document}

\twocolumn[

\aistatstitle{A Framework for Sample Efficient Interval Estimation  with Control Variates}

\aistatsauthor{ Shengjia Zhao \And Christopher Yeh \And  Stefano Ermon }

\aistatsaddress{ Stanford University \And Stanford University \And Stanford University } ]

\begin{abstract}
We consider the problem of estimating confidence intervals for the mean of a random variable, where the goal is to produce the smallest possible interval for a given number of samples. While minimax optimal algorithms are known for this problem in the general case, improved performance is possible under additional assumptions. In particular, we design an estimation algorithm to take advantage of side information in the form of a control variate, leveraging order statistics. Under certain conditions on the quality of the control variates, we show improved asymptotic efficiency compared to existing estimation algorithms. Empirically, we demonstrate superior performance on several real world surveying and estimation tasks where we use the output of regression models as the control variates.
\end{abstract}

\section{Introduction}

Many real world problems require estimation of the mean of a random variable from unbiased samples. In high risk applications, the estimation algorithm should output a confidence interval, with the guarantee that the true mean belongs to the interval with e.g. 99\% probability. The classic tools for this task  are concentration inequalities, such as the Chernoff, Bernstein, or Chebychev inequalities~\citep{hoeffding1962probability,bernstein1924modification,vershynin2010introduction}.

However, for many tasks, obtaining unbiased samples is expensive. For example, when estimating demographic quantities, such as income or political preference, drawing unbiased samples can require field survey. To make the situation worse, often we seek more granular estimates, e.g. for each individual district or county, meaning that we need to draw a sufficient number of samples for every district or county. Another difficulty can arise when we need high accuracy (i.e. a small confidence interval), since confidence intervals produced by typical concentration inequalities have size $O(1/\sqrt{\mathrm{number\ of\ samples}})$. In other words, to reduce the size of a confidence interval by a factor of 10, we need 100 times more samples.

This dilemma is generally unavoidable because concentration inequalities such as Chernoff or Chebychev are minimax optimal: there exist distributions for which these inequalities cannot be improved. No-free-lunch results~\citep{van2000asymptotic} imply that any alternative estimation algorithm that performs better (i.e. outputs a confidence interval with smaller size) on some problems, must perform worse on other problems. Nevertheless, we can identify a subset of problems where better estimation algorithms are possible. 

We consider the class of problems where we have side information. We formalize side information as a random variable with known expectation and whose value is close (with high probability) to the random variable we want to estimate. Following the terminology in the Monte Carlo simulation literature, we call this side information random variable a ``control variate''~\citep{lemieux2017controlvariates}. Instead of the original estimation task, we can estimate the expected difference between the original random variable and the control variate. The hope is that the distribution of this difference is concentrated around $0$.
Some estimation algorithms output very good (small sized) confidence intervals for distributions concentrated around $0$ compared to classic methods such as Chernoff bounds. 

Many practical problems have very good control variates. One important class of problems is when we have a predictor for the random variable, and we can use the prediction as the control variate.
For example, if we would like to estimate a neighborhood's average voting pattern, then we might have a prediction function for political preference from Google Street View images of that neighborhood~\citep{gebru13108streetview}; if we would like to estimate the average asset wealth of households in a certain geographic region, we might have a regressor that predicts income from satellite images~\citep{jean2016combining}. These classifiers could be trained on past data (e.g. previous year survey results) or similar datasets, or they could even be crafted by hand. 

For these problems we propose concentration bounds based on order statistics. In particular, we draw a connection between the recently proposed concept of \emph{resilience}~\citep{steinhardt2018robust} and concentration bounds on order statistics. We show how to use these concentration inequalities to design estimation algorithms that output better confidence intervals when we have a good control variate (e.g. output confidence intervals of size $O(1/\mathrm{number\ of\ samples})$). 

Our proposed estimation algorithm always produces valid confidence intervals, i.e. the true mean belongs to the interval with a specified probability. The only risk is that when the control variate is poor, the confidence interval could be worse (larger) than classic baselines such as Chernoff inequalities. We empirically show superior performance of the proposed estimation algorithm on three real world tasks: bounding regression error, estimating average wealth with satellite images, and estimating the covariance between wealth and education level. 
\newcommand{\mean}{{\mu_{\mathrm{mean}}}}
\newcommand{\tY}{{\tilde{Y}}}
\newcommand{\ty}{{\tilde{y}}}

\section{Problem Setup}

Our objective is to estimate the mean of some random variable $Y$ taking values in $\Yc \subseteq \R^d$. Given i.i.d. samples $y_1, \dotsc, y_m \sim Y$ and some choice of confidence level $\zeta \in (0, 1)$, an estimation algorithm outputs $\hat{\mu} \in \R^d$ and an confidence interval size $c \in \R^+$.
The estimation algorithm must satisfy
\begin{equation} \Pr\left[ \lVert \hat{\mu} - \E[Y] \rVert > c \right] \leq \zeta \label{eq:problem_def} \end{equation}
where the probability $\Pr$ is with respect to the random samples of $y_1, \dotsc, y_m$ and any additional randomness in the execution of the (randomized) algorithm, and $\lVert \hat{\mu} - \E[Y] \rVert$ is any choice of semi-norm.

We will first focus on one dimensional problems where $Y \in \mathbb{R}$, and choose the norm $\lVert \hat{\mu} - \E[Y] \rVert = \lvert \hat{\mu} - \E[Y] \rvert$, then extend several results to more general setups. 

\subsection{Baseline Estimators and Optimality}
\label{sec:baseline}

The classical approach is to estimate $\E[Y]$ with the empirical mean $\mean = \frac{1}{m}\sum_{i=1}^m y_i$, and its estimation error $ \lvert \mean - \E[Y] \rvert$ 
can be controlled using 
concentration inequalities.

To obtain concentration inequalities, we need some assumptions on $\Yc$. For example, if there is some very small probability that $Y = +\infty$, then $\E[Y]$ is unbounded, but $\mean$ can be finite, which makes it impossible to bound $\lvert \mean - \E[Y] \rvert$. Common assumptions include sub-Gaussian, bounded moments, sub-exponential, etc~\citep{vershynin2010introduction}. We will consider the sub-Gaussian and bounded moments assumptions in this paper.

\textbf{Sub-Gaussian}: If we assume $\forall t \in \R,\ \E[e^{t(Y-\Eb[Y])}] \leq e^{\sigma^2 t^2/2}$, then Chernoff-Hoeffding is the classic concentration inequality for confidence interval estimation~\citep{hoeffding1962probability}
\begin{align*}
    \Pr\left[ \lvert \hat{\mu} - \E[Y] \rvert > \sqrt{\frac{2\sigma^2}{m}\log \frac{2}{\zeta}} \right] \leq \zeta.
\end{align*}

In particular, when $Y$ is supported on $[a, b]$ we have 
\begin{align*}
    \Pr\left[ \left\lvert \hat{\mu} - \E[Y] \right\rvert > \sqrt{\frac{(b-a)^2}{2m}\log \frac{2}{\zeta}} \right] \leq \zeta.
\end{align*}

\textbf{Bounded Moments}: Another common assumption is that the $k$-th order moment of $Y$ is bounded, i.e. $\E[\lvert Y - \E[Y] \rvert^k] \leq \sigma^k$ for some $\sigma > 0$. 

Under bounded moment assumptions, several concentration inequalities are known, such as the Chebychev inequality, Kolmogorov inequality~\citep{hajek1955generalization} and Bernstein inequality~\citep{bernstein1924modification}. For example, when $Y$ has a bounded second order moment, the Chebychev inequality states the following:
\begin{equation}
    \Pr\left[ \left\lvert \hat{\mu} - \E[Y] \right\rvert \geq \frac{\sigma}{\sqrt{\zeta m}} \right] \leq \zeta
    \label{eq:chebychev}
\end{equation}

It is known that these bounds are (asymptotically in $m$) minimax optimal. There exist random variables $Y$ that satisfy the respective assumptions of each inequality, and the inequality cannot be improved~\citep{hoeffding1962probability}. Therefore, to further improve these estimation algorithms, additional assumptions will be necessary.

\subsection{Control Variates}\label{sec:control_variates}

Suppose $\tY$ is another random variable jointly distributed with $Y$ and that we know its mean $\E[\tY]$ (or have a very accurate estimate of it). 
We also have samples drawn from their joint distribution $(y_1, \ty_1), \dotsc, (y_m, \ty_m) \sim Y, \tY$. 

For $\tY$ to be useful for our application, its value needs to be close to $Y$. 
In other words, $Y - \tY$ should be a random variable that is concentrated around $0$. The purpose of this variable $\tY$ is 
similar to 
a ``control variate'' in the Monte Carlo community, but we use it here for a different task of interval estimation. 

For example, in our household wealth estimation example, let $Y$ denote the household-level wealth in a randomly sampled village, and let $\tY$ be the predicted household-level wealth based on the satellite image of that village. If our predictor is accurate (i.e. $Y \approx \tY$ with high probability), then $\tY$ could be an effective control variate for $Y$. In addition, we can also estimate $\E[\tY]$ very accurately because a very large number of satellite images 
are available with little cost~\citep{wulder2012landsat}, so obtaining samples from $\tY$ (without the corresponding $Y$) is inexpensive. 

We would like to design an estimation algorithm that takes as input the samples $(y_1, \ty_1), \dotsc, (y_m, \ty_m)$ and the known value of $\E[\tY]$, and outputs a $\hat{\mu}$ along with an estimation error $c \in \R^+$ such that Eq.(\ref{eq:problem_def}) holds. The hope is that because of the additional information $\tY$ we are no longer restricted by the mini-max bounds on the confidence interval size $c$, and can output confidence intervals of smaller size.   

\newcommand{\stuff}{{ \left(\frac{(m+1)^k}{\zeta}\right)^{\frac{1}{m-k}}}} 
\newcommand{\stuffn}{{ \left(\frac{\zeta}{(m+1)^k}\right)^{\frac{1}{m-k}}}} 
\newcommand{\vdual}{{\lVert v \rVert_*}}
\newcommand{\bare}{{\bar{\epsilon}}}
\newcommand{\tF}{{\tilde{F}}}
\newcommand{\tZ}{{\tilde{Z}}}
\newcommand{\zmean}{{\bar{Z}}}

\section{Control Variates Interval Estimation}

\subsection{Estimation by Order Statistics}
\label{sec:estimation_procedure}

For convenience we denote $Z = Y - \tY$ and $z_i = y_i - \ty_i$. Given the samples $z_1, \dotsc, z_m$ we can compute their order statistics, which are the $m$ samples in sorted order such that $Z_{(1)} \leq \dotsb \leq Z_{(m)}$.

We first study the one dimensional real valued case (i.e. $Z \in \Zc = \R$).



 
Our control variate estimation algorithm consists of three steps: given input samples $z_1, \dotsc, z_m$, $\E[\tY] \in \mathbb{R}$ and desired confidence $\zeta \in (0, 1)$
 
\textbf{1)} Choose some value of $k \in \{0, \dotsc, m-1\}$. Set
\begin{equation}
    \hat{\mu} = \E[\tY] + \frac{Z_{(m-k)} + Z_{(1+k)}}{2}
    \label{eq:hat_mu}
\end{equation}

\textbf{2)} Let $r$ be the smallest value such that the following concentration inequalities are true: 
\begin{align*}
    \Pr\left[Z_{(m-k)}  < \E[Z] - r\right] &\leq \zeta/2 \\ 
    \Pr\left[Z_{(1+k)}  > \E[Z] + r\right] &\leq \zeta/2
    \numberthis\label{eq:order_bound}
\end{align*}
Algorithm~\ref{alg:bound} computes such a value of $r$. 

\textbf{3)} Output $\hat{\mu}$ as the estimate of $\Eb[Y]$ and $r + \frac{Z_{(m-k)} - Z_{(1+k)}}{2}$ as the confidence interval size.

The following proposition guarantees the correctness of the control variate estimation algorithm.

\begin{restatable}{proposition}{orderbound}
\label{prop:validity}
Let $\hat{\mu}$ be defined as in Eq.(\ref{eq:hat_mu}). If Eq.(\ref{eq:order_bound}) is satisfied for some $\zeta \in (0, 1)$ and $c > 0$, then
\begin{align*}
    \Pr\left[\lvert \hat{\mu} - \E[Y] \rvert > r + \frac{Z_{(m-k)} - Z_{(1+k)}}{2} \right] \leq \zeta. \numberthis\label{eq:concentration_order}
\end{align*}
\end{restatable}

\begin{proof}[Proof of Proposition~\ref{prop:validity}] See Appendix. 
\end{proof}

Note that in Eq.(\ref{eq:concentration_order}), the confidence interval size consists of two parts: $\frac{Z_{(m-k)} - Z_{(1+k)}}{2}$ and $c$. We will show that $c$ is much smaller (in fact, has asymptotically better rates) compared to baselines previously discussed in Section~\ref{sec:baseline}. If we have a good control variate $\tY$ (close to $Y$), then $\frac{Z_{(m-k)} - Z_{(1+k)}}{2}$ will be small. On the other hand, if the control variate $\tY$ is poor, this algorithm could produce worse (larger) confidence intervals. This trade-off is unavoidable because of ``no-free-lunch results''~\citep{van2000asymptotic}.

\subsection{Order Statistics Concentration Inequalities}
\label{sec:order_bounds}

We will now show several bounds on the order statistics in Eq.(\ref{eq:order_bound}). There are several known bounds~\citep{de2007extreme,castillo2012extreme} that can be applied to satisfy Eq.(\ref{eq:order_bound}). However, these bounds are distribution dependent: we must assume that $Z$ either is distributed as some known distribution (e.g. Gaussian), or asymptotically converges to some fixed distribution (e.g. Gumbel, Weibull or Frechet using the Fisher–Tippett–Gnedenko theorem) when $m$ is very large~\citep{fisher1928limiting}. 

Our contribution is to associate bounds on the order statistics with the notion of resilience used in the robust statistics literature. Many distribution independent conditions for resilience are known, which imply \emph{distribution independent bounds on the order statistics}.


We first reproduce the definition of resilience from \citet{steinhardt2018robust} with a small modification.
\begin{definition}
\label{def:resilience}
Let $s: [0, 1] \to \R^+$ be any function. We say a random variable $Z \in \Zc$ is $s$-resilient from above if for any $\epsilon \in [0, 1]$ and measurable $B \subseteq \Zc$ such that $\Pr[Z \in B] \geq 1 - \epsilon$, we have
\[
    \E[Z|Z\in B] - \E[Z] \leq s(\epsilon).
\]
It is $s$-resilient from below if
\[
    \E[Z] - \E[Z|Z\in B] \leq s(\epsilon).
\]
We say that $Z$ is $s$-resilient if it is both $s$-resilient from above and $s$-resilient from below. 
\end{definition}

When a random variable is $s$-resilient, we are essentially bounding the probability that $Z$ takes a value much larger (or smaller) than $\E[Z]$. To see this, suppose for some $\epsilon \in (0, 1)$, we choose $B = \lbrace Z: Z > \E[Z] + s(\epsilon) \rbrace$, then $\E[Z\vert Z \in B] - \E[Z] > s(\epsilon)$, then by our requirement of resilience, it must be that $\Pr[Z \in B] < 1 - \epsilon$. Similar to other types of assumptions such as sub-Gaussian or bounded moments, resilience is also an assumption on how much a random variable can differ from its expectation. What is special about resilience is that it is particularly convenient for showing bounds on the order statistics.    


Without loss of generality, we also assume that $s(\epsilon)$ is monotonically non-decreasing. If a random variable is $\hat{s}$-resilient, but $\hat{s}$ is not monotonically non-decreasing, we can define $s(\epsilon) = \inf_{\epsilon \leq \epsilon' \leq 1} \hat{s}(\epsilon')$, which is monotonically non-decreasing. It is easy to show that $Z$ is $s$-resilient if and only if it is $\hat{s}$-resilient.

Many assumptions on the random variable $Z$ can be converted into assumptions on resilience. For example, in the following Lemma~\ref{lemma:example_resilient}, (1) is trivial to show, (2) is proved in \citep{steinhardt2018robust}, and (3) we prove in the appendix. 
\begin{restatable}{lemma}{resilient}
\label{lemma:example_resilient}
The following random variables are resilient:
\begin{enumerate}
    \item \textbf{Bounded:} If $\Zc \subseteq [a, b]$, then $Z$ is $(b-\E[Z])$-resilient from above and $(\E[Z]-a)$-resilient from below. It is $(b-a)$-resilient.
    \item \textbf{Bounded Moments:} If $\E\left[ |Z-\E[Z]|^l \right] \leq \sigma^l$ for some $l > 1$, then $Z$ is $s$-resilient for $s(\epsilon)=\frac{\sigma}{(1-\epsilon)^{1/l}}$. 
    \item \textbf{Sub-Gaussian:} If $Z - \E[Z]$ is $\sigma^2$ sub-Gaussian, then $Z$ is $s$-resilient for $s(\epsilon)=\sqrt{2\sigma\log\frac{1}{1 - \epsilon}} + \sqrt{2\pi}\sigma$.
\end{enumerate}
\end{restatable}

Given the definition and examples of resilience, we can now use the following theorem to provide bounds on the order statistics when resilience holds.

\begin{restatable}{theorem}{singledim}
\label{thm:resilient}
Let $Z_{(1)}, \dotsc, Z_{(m)}$ denote the order statistics of $m$ independent samples of a random variable $Z$.
If $Z$ is $s$-resilient, $\forall \zeta \in (0, 1), T \in \mathbb{N}$ and $0 \leq k < m/2$, then letting $r$ be the output of Algorithm~\ref{alg:bound}, we have
\begin{align}
    \Pr\left[Z_{(m-k)} \leq \E[Z] - r\right] &\leq \zeta \label{eq:upper_bound} \\
    \Pr\left[Z_{(1+k)} \geq \E[Z] + r\right] &\leq \zeta \label{eq:lower_bound}
\end{align}
If $Z$ is $s$-resilient from above/below, then only Eq.(\ref{eq:upper_bound})/Eq.(\ref{eq:lower_bound}) holds.
\end{restatable}
\begin{proof}[Proof of Theorem~\ref{thm:resilient}] See Appendix.
\end{proof}

Note that the algorithm has an additional hyper-parameter $T$ which is the number of iterations. We can choose any value of $T \in \mathbb{N}$, but choosing a large $T$ always leads to a better confidence interval compared to choosing a smaller $T$. We observe in the experiments that choosing any $T \geq 10$ is optimal (up to floating point errors). Note that the choice of $T$ does not affect the asymptotic performance in Corollary~\ref{cor:asymptote}. 

\begin{algorithm}
Input: sample size $m \in \mathbb{N}^+$; order $k < m/2$; confidence $\zeta \in (0, 1)$; resilience $s: (0, 1) \to \R^+$; and number of iterations $T \in \mathbb{N}$
\begin{algorithmic}
\State Set $v_0 = \stuffn$
\For {$i=1, \dotsc, T$}
\State Set $v_i= \left(\frac{\zeta}{(m(1-v_{i-1})+1)^k}\right)^{\frac{1}{m-k}}$ 
\EndFor
\State Set $r = s(v_T) (v_T^{-1} - 1)$
\State Return $v_T$ and $r$
\end{algorithmic}
\caption{Order Statistics Confidence Interval}
\label{alg:bound}
\end{algorithm} 

Theorem~\ref{thm:resilient} involves expressions with good constants. This obscures the asymptotic performance of the bound, which we reveal in the following corollary.

\begin{restatable}{corollary}{asymptote}
\label{cor:asymptote}
If $Z$ is $b_1(1-\epsilon)^a + b_2$ resilient for any constants $a \in [-1, 0]$ and $b_1, b_2 \in \R^+$, then there exists $\lambda > 0$ such that for sufficiently large $m$ 
\begin{align*}
    \Pr&\left[\E[Z] \leq Z_{(1+k)} - \lambda\frac{\log \frac{1}{\zeta} + k\log m}{m^{1+a}} \right] \leq \zeta \\
    \Pr&\left[\E[Z] \geq Z_{(m-k)} +  \lambda\frac{\log \frac{1}{\zeta} + k\log m}{m^{1+a}} \right] \leq \zeta.
\end{align*}
\end{restatable}
\begin{proof}[Proof of Corollary~\ref{cor:asymptote}] See Appendix.
\end{proof}



Based on different assumptions on resilience in Lemma~\ref{lemma:example_resilient}, we can further simplify Corollary~\ref{cor:asymptote} and obtain more concrete bounds in Section~\ref{sec:comparison}.

Finally, in Lemma~\ref{lemma:example_resilient}, a random variable $Z$ bounded in $[a, b]$ is $s$-resilient, but $s$ depends on $\E[Z]$ which is unknown. Therefore, we cannot evaluate $s$ in Algorithm~\ref{alg:bound}. One option is to use the weaker conclusion that $Z$ is $(b-a)$-resilient and obtain a bound with worse constants (we choose this option in our asymptotic rate analysis in Section~\ref{sec:comparison}).

In practice it is possible to compute a bound with better constants: we only compute $v_T$ in Algorithm~\ref{alg:bound} which does not depend on $s$, and use the following improved bound. 
\begin{restatable}{corollary}{boundedexample}
\label{cor:bounded_example}
Let $Z_{(1)}, \dotsc, Z_{(m)}$ denote the order statistics of $m$ independent samples of a random variable $Z$. If $Z$ is bounded in $[a, b]$, $\forall \zeta \in (0, 1), T \in \mathbb{N}$ and $0 \leq k < m/2$, then letting $v_T$ be computed as in Algorithm~\ref{alg:bound}, we have
\begin{align*}
    \Pr\left[ \E[Z] \leq a + v_T(Z_{(1+k)} - a) \right] & \leq \zeta \\
    \Pr\left[ \E[Z] \geq b - v_T(b - Z_{(m-k)}) \right] & \leq \zeta.
\end{align*}
\end{restatable}
\begin{proof}[Proof of Corollary~\ref{cor:bounded_example}] See Appendix.
\end{proof}
Instead of the standard estimation procedure in Section~\ref{sec:control_variates}, we directly output 
$$
    \left[a + v_T(Z_{(1+k)} - a),\ b - v_T(b - Z_{(m-k)})\right]
$$
as our confidence interval for $\E[Z]$ (that holds with $1-2\zeta$ probability) in the experiments.



\subsection{Multi-Dimensional Extension}

We will extend the above results to multi-dimensional estimation problems. We first provide a multi-dimensional definition of resilience that extends Definition~\ref{def:resilience}.

\begin{definition}
Let $Z$ be a random variable on $\Zc \subseteq \mathbb{R}^d$, and $\lVert \cdot \rVert, \lVert \cdot \rVert_*$ be a pair of dual norms on $\mathbb{R}^d$. We say $Z$ is $s$-resilient if $\forall v \in \mathbb{R}^d$ such that $\vdual = 1$, and for all measurable $B \subseteq \Zc$ such that $\Pr[B] \geq 1 - \epsilon$, we have
$$
    \Eb[\langle Z, v \rangle \vert B] - \Eb[\langle Z, v \rangle ] \leq s(\epsilon).
$$
\end{definition}
As before, resilience in multiple dimensions also implies concentration bounds on the order statistics. The following theorem is the analog of Theorem~\ref{thm:resilient}. 

\begin{restatable}{theorem}{multidim}
\label{thm:multi-dim}
Let $Z_{(1)}, \dotsc, Z_{(m)}$ be independent samples of $Z$ ordered such that 
$\lVert Z_{(1)} \rVert \leq \dotsb \leq \lVert Z_{(m)} \rVert$.
If $Z$ is $s$-resilient, then for any $r$ output by Algorithm~\ref{alg:bound} and for any $\zeta \in (0, 1)$, we have
$$
    \Pr\left[\lVert Z_{(m-k)} \rVert \leq \lVert \E[Z] \rVert - r \right] \leq \zeta.
$$
\end{restatable}
\begin{proof}[Proof of Theorem~\ref{thm:multi-dim}] See Appendix.
\end{proof}


\subsection{Rate Comparison}
\label{sec:comparison}
Table~\ref{table:rate} summarizes the asymptotic performance of our method compared to the baselines. Even though we consider the low sample setup (i.e. $m$ is small), these asymptotic rates still provide insight into the trade-off between different methods.

In particular, as shown in Proposition~\ref{prop:validity} our method has two terms that determine the confidence interval: $r$ and $\frac{Z_{(m-k)} - Z_{(1+k)}}{2}$. 
\begin{align*}
    \Pr\left[\lvert \hat{\mu} - \E[Y] \rvert > r + \frac{Z_{(m-k)} - Z_{(1+k)}}{2} \right] \leq \zeta.
\end{align*}
The latter term we will denote as $B$ and it is determined by the quality of the control variate. For baseline methods there is only a single term $c$ that determines the size of the confidence interval:
\begin{align*}
    \Pr\left[\lvert \hat{\mu} - \E[Y] \rvert > c \right] \leq \zeta.
\end{align*}
In Table~\ref{table:rate}, we show that under each class of assumptions, $r$ in our proposed algorithm always has a better rate compared to $c$ (i.e. it is smaller when $m, \zeta$ are sufficiently large). Whether the improvement can justify the additional term $B$ determines whether our algorithm performs well in practice. 

\begin{table*}
\begin{center}
\begin{tabular}{c|c|c}
    Conditions on $Z$ & Bounded in $[a, b]$ & Finite $\Eb[Z^2]$   \\
    \hline
    Mean $\frac{1}{m} \sum_i Z_i$
        & $\Theta\left( \sqrt{\frac{1}{m}\log \frac{1}{\zeta}} \right)$ (Chernoff)
        & $\Theta\left( \frac{1}{\sqrt{m \zeta}}\right)$ (Chebyshev) \\
    \hline
    Maximum (Minimum) $Z_{(m)}, Z_{(1)}$
        & $B + O\left( \frac{1}{m}\log \frac{1}{\zeta} \right)$
        & $B + O\left( \frac{1}{\sqrt{m}}\log \frac{1}{\zeta} \right)$ \\
    \hline
    $k$-th largest (smallest) $Z_{(m-k)}, Z_{(1+k)}$
        & $B + O\left( \frac{k\log m}{m} \log \frac{1}{\zeta} \right)$
        & $B + O\left( \frac{k\log m}{\sqrt{m}} \log \frac{1}{\zeta} \right)$
\end{tabular}
\end{center}
\caption{Summary of asymptotic size of the confidence interval for estimation algorithms using different concentration inequalities. $B$ is some value that corresponds to the bias of our method because of the $\frac{Z_{(m-k)} - Z_{(1+k)}}{2}$ term in Proposition~\ref{prop:validity}.}
\label{table:rate}
\end{table*}







\section{Related Work}

Extreme value theory~\citep{de2007extreme,castillo2012extreme} studies the probability of rare events or large deviation. Most results are asymptotic~\citep{de2007extreme} and are not applicable to our setup assuming small sample size. Several non-asymptotic results are also used in our proofs such as Eq.(\ref{eq:order_stats}).

The notion of resilience~\citep{steinhardt2017resilience,steinhardt2018robust} is most commonly used in analyzing robust estimation. Our paper draws the connection between resilience and order statistics concentration bounds. We hope this connection can be further exploited in future work to transfer results between the two fields of research.  

A line of research related to ours is semi-supervised transfer learning and domain adaptation~\citep{daume2010frustratingly, donahue2013semi, kumar2010co, ding2018semi, lopez2012semi, Saito2019}. In both setups, we have a pretrained classifier or regressor; in the target domain, there is a small amount of labeled data. The difference is in the objective: domain adaptation use the labeled data to fine-tune our classifier or regressor, while our objective is confidence interval estimation on the target domain. The different objectives lead to different sets of tools and desiderata. 


\section{Experiments}



\subsection{Certifying Regression Performance}
\label{sec:exp1}
\begin{figure*}
    \centering
    \begin{tabular}{cc}
    \includegraphics[width=0.38\linewidth]{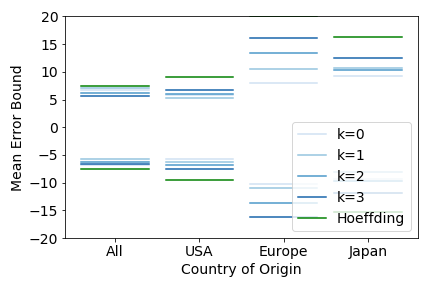} &  \includegraphics[width=0.38\linewidth]{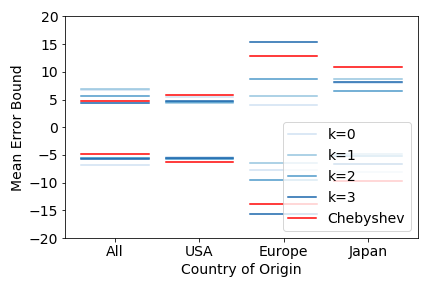} \\
    \includegraphics[width=0.38\linewidth]{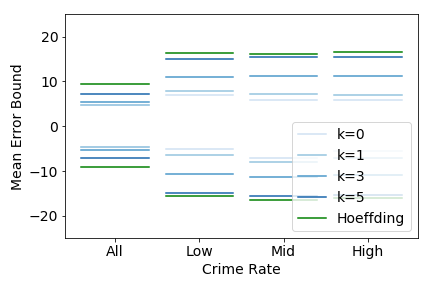} & \includegraphics[width=0.38\linewidth]{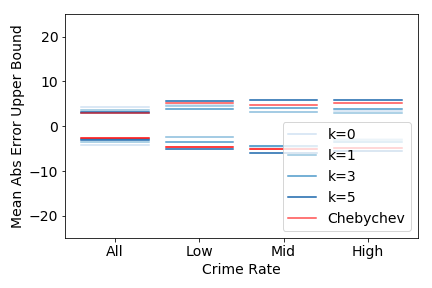} \\
    \end{tabular}
    \caption{Confidence intervals (both upper and lower bounds) on $\E[Z]$ from different estimation algorithms. For our algorithm, we try different values of $k$ (the $k$-th largest / smallest). For the MPG dataset, we also evaluate the regression error conditioned on different country of origin (top). For housing dataset, we evaluate the regression error conditioned on different crime rate level (bottom). In most of the experiments, our estimation algorithm performs better (outputs smaller confidence intervals). Chebychev sometimes performs better, especially with large sample size. Hoeffding always performs worse in this setup.}
    \label{fig:mpg}
\end{figure*}

\begin{figure*}
\centering
\includegraphics[width=0.8\linewidth]{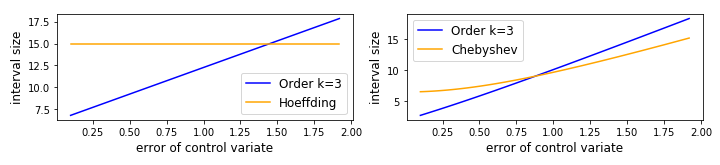}
\caption{Size of the confidence interval as a function of the error of the control variate. We scale up (>1.0) or scaled down (<1.0) the error (i.e. $Z_{\mathrm{new}} = \alpha Z_{\mathrm{old}}$ for $\alpha \in [0, 2]$). When the control variate has  smaller error the confidence interval is significantly better compared to Hoeffding or Chebychev.}
\label{fig:bound_scale}
\end{figure*}

Our first task is to upper and lower bound the difference between the output of a regression function $\tY$ and the true target attribute $Y$. For this task, our goal is to bound the expected error $Z = Y - \tY$ of the regression function. We are not directly interested in the expected value of the target attribute $\E[Y]$; instead we only want to show bounds $\mathrm{LB} \leq \E[Z] \leq \mathrm{UB}$. In addition, instead of a single global accuracy, we might care about accuracy for sub-groups in the data (e.g., based on some feature or sensitive attribute). In other words, let $U$ be some random variable taking a finite set of values, we want to know  the regression error $\E[Z \vert U=u]$ for each value of $U=u$. This can be important for fairness or identifying particular failure cases. 

If the regression function is accurate, $Z$ should be concentrated around $0$, making it feasible to obtain better bounds with order statistics. We compare the bounds of Section~\ref{sec:order_bounds} with the baseline bounds of Section~\ref{sec:baseline}. Code is available at https://github.com/ermongroup/ControlVariateBound


\textbf{Datasets}: We use two classic regression datasets in the UCI repository~\citep{asuncion2007uci}: Auto MPG, where the task is to predict the miles per gallon (MPG) of a vehicle based on 10 features; and Boston housing price prediction, where the task is to predict the housing price from 13 features. 

\textbf{Assumptions}: As explained in Section~\ref{sec:baseline}, all estimation algorithms require some assumptions. Here, we either assume bounded support or bounded variance. Optimal choice of these assumptions usually relies on domain knowledge about the problem and is beyond the scope of this paper. Here, we simply assume that the error is bounded by $\pm b/2$ where $b$ is the maximum MPG in the entire dataset. The reason is that any regression algorithm can trivially output $b/2$ and achieve this error. In the bounded variance case, we first compute an upper bound on $\E[Z^2]$ that holds with $1-\zeta/2$ probability by Hoeffding inequality as an upper bound on the variance. 

\textbf{Results:} The results are shown in Figure~\ref{fig:mpg}. Our order statistics bound works better in general if the number of test samples $m$ is small. Here both datasets contain approximately 100 test samples, and our bound performs on-par with Chebychev and better than Hoeffding. Our bound also performs better when the control variate is more accurate (i.e. $Z$ is concentrated around zero). This is empirically verified in  Figure~\ref{fig:bound_scale}. 

Our bound also depends on the choice of $k$. In general, with more data choosing a larger value of $k$ is preferable, and vice versa.

\begin{figure}[h]
    \centering
    \includegraphics[width=0.45\linewidth]{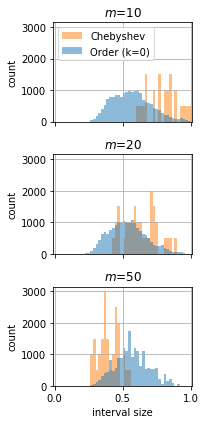}
    \includegraphics[width=0.45\linewidth]{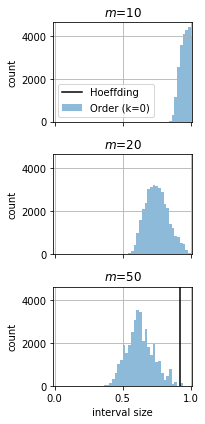}
    \caption{Histograms of 99\%-confidence interval sizes for the average household-level asset wealth within a province, for 36 provinces across 13 countries in DHS surveys of sub-Saharan Africa. 1000 random subsets of size $m$ are sampled from each province. We assume that household-level asset wealth is a random variable either with finite variance (left), or bounded in $[0, 1]$ (right). 
    The histograms show interval sizes pooled over all 36 test provinces.
    For small sample size per province (i.e. <20 samples per province) our method achieves smaller confidence intervals. With more samples, the bias of our estimation algorithm dominates and our method performs worse.}
    \label{fig:poverty_mean}
\end{figure}





\subsection{Poverty Estimation Task}
\label{sec:exp2}

We apply our estimation algorithm on a real world task, where we estimate the average household-level asset wealth across provinces of countries in 23 sub-Saharan African countries. We used DHS Survey collected between 2009-16 and constructed an average household asset wealth index for 19,669 villages following the procedure described in \cite{jean2016combining}.

\textbf{Setup}: We emulate the setup where we have survey results from several countries, and train a regressor (a convolutional neural network) to predict asset wealth from satellite images (multispectral Landsat and nighttime VIIRS). To estimate the average asset wealth for a new country, we apply this regressor to satellite images 
from that country; we use the output of the regressor as a control variate. 
More specifically, we randomly pick 80\% of the countries to train the regressor, and test the performance of our estimation algorithm on the remaining countries. We also use cross validation to more accurately evaluate our performance. 

\textbf{Assumptions}: As in Section~\ref{sec:exp1}, for estimation with bounded random variables, we upper- and lower-bound household wealth by the maximum and minimum wealth across the entire dataset. For estimation with bounded moment random variables, we use the empirical standard deviation estimated across the entire country, multiplied by an additional margin of 1.5$\times$. Because of the small sample size, Chernoff bounding the standard deviation as in Section~\ref{sec:exp1} is infeasible.  

\textbf{Results}: The results are shown in Figure~\ref{fig:poverty_mean}. Although our regression model is trained on all 23 countries in our dataset, we only test our method on the 36 provinces across 13 countries from which we have at least 90 labeled survey examples. Compared to Chernoff bound or Chebychev our estimation algorithm perform better when sample size is small, and worse when sample size is large. Because of the difficulty of predicting wealth from satellite images, the control variate is not very accurate, and further improvements are possible with improved prediction accuracy.


\subsection{Covariance Estimation}
The DHS surveys also include other demographic variables besides household asset wealth. Policy makers may be interested in the covariance of these demographic quantities, such as between maternal education level and household asset wealth.

More formally let $W$ be the random variable that represents the average level of maternal education in a village, and $U$ represent the average household asset wealth in the same village. The random variable we actually want to estimate is $Y = UW - U\E[W]$ because  
\[
    \E[Y] = \E[UW] - \E[U]\E[W] = \mathrm{Cov}(U, W) 
\]
We assume that $W$ is a quantity that is easy to survey, or has available data, so $\E[W]$ is known. This is commonly the case, when certain demographic variables are more widely surveyed than others. As before, we can train a regressor to predict $U$ from satellite images, and we denote it's output as $\tilde{U}$. Our control variate is then $\tY = \tilde{U}W - \tilde{U}\E[W]$. By using these new definitions for $Y$ and $\tY$ we can apply our estimation algorithm.

\textbf{Setup}: The setup is identical to Section~\ref{sec:exp2}, where we train the asset wealth regression model on 80\% of the countries and test our estimation algorithm on the remaining countries. However, we only have maternal education survey data on 9 countries, so we only estimate confidence intervals of covariance in these countries. Maternal education level is measured at each household on an integer scale from 0 to 3, then averaged within each village. All the other assumptions are also identical to before.

\textbf{Result}: The results are shown in Figure~\ref{fig:poverty_cov}. As expected, our estimation algorithm achieves superior performance compared to baseline estimators when the sample size $m$ is small.




\begin{figure}
    \centering
    \includegraphics[width=0.45\linewidth]{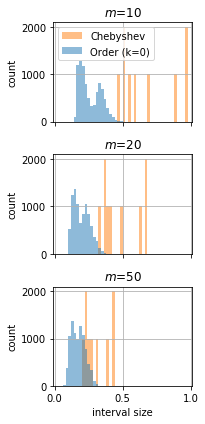}
    \includegraphics[width=0.45\linewidth]{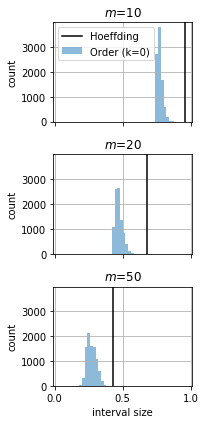}
    \caption{Histograms of confidence interval sizes at $\zeta = 0.01$ for the covariance between maternal education and household asset wealth, for 9 countries in DHS surveys of sub-Saharan Africa. 1000 random subsets of size $m$ are sampled from each province. The confidence interval derived from order statistics outperforms both the Hoeffding interval and the Chebychev interval.}
    \label{fig:poverty_cov}
\end{figure}
\section{Conclusion}

In this paper we propose a framework for estimating the confidence interval given a control variate random variable as side information. We show that under certain conditions on the control variate, the estimation algorithms out-performs classic minimax optimal estimation algorithms both asymptotically and empirically.

A major weakness of the estimator is diminished performance when we have a large number of samples. Because of no-free-lunch results, trade-offs are unavoidable, but it is an interesting direction of future research to find either better trade-offs or prove its impossibility.

\section{Acknowledgments}

This research was supported by TRI, Intel, NSF (\#1651565, \#1522054, \#1733686), and ONR. The survey datasets are generously provided by Marshall Burke at Stanford University. 

\bibliography{main}

\newpage
\newpage
\section{Proofs}

\orderbound*
\begin{proof}[Proof of Proposition~\ref{prop:validity}]
Whenever $\E[Z] \leq Z_{(m-k)} + r$ and $\E[Z] \geq  Z_{(1+k)} - r$, we have
\begin{align*}
    &\lvert \hat{\mu} - \E[Y] \rvert = \left\lvert \E[Z] - \frac{Z_{(m-k)} + Z_{(1+k)}}{2} \right\rvert \\
    &= \max\left(\E[Z] - \frac{Z_{(m-k)} + Z_{(1+k)}}{2}, \right. \\ &\qquad\qquad \left. \frac{Z_{(m-k)} + Z_{(1+k)}}{2} - \E[Z] \right) \\
    &\leq \max\left(Z_{(m-k)} + r - \frac{Z_{(m-k)} + Z_{(1+k)}}{2}, \right. \\ &\qquad\qquad \left. \frac{Z_{(m-k)} + Z_{(1+k)}}{2} - Z_{(1+k)} + r \right) \\
    &= \frac{Z_{(m-k)} - Z_{(1+k)}}{2} + r
\end{align*}
In other words, we cannot have 
\[\lvert \hat{\mu} - \E[Y] \rvert > \frac{Z_{(m-k)} - Z_{(1+k)}}{2} + r \]
unless either $\E[Z] >Z_{(m-k)} + r$ or $\E[Z] <  Z_{(1+k)} - r$ is true, which implies that
\begin{align*}
    &\Pr\left[\lvert \hat{\mu} - \E[Y] \rvert > r + \frac{Z_{(m-k)} - Z_{(1+k)}}{2} \right] \\
    &\leq \Pr[\E[Z] > Z_{(m-k)} + r \text{ or } \E[Z] < Z_{(1+k)} - r] \\
    &\leq \Pr[\E[Z] > Z_{(m-k)} + r] + \Pr[\E[Z] < Z_{(1+k)} - r] \\
    &\leq \zeta
\end{align*}
\end{proof}

\singledim*
\begin{proof}[Proof of Theorem~\ref{thm:resilient}]
Before proving Theorem~\ref{thm:resilient}, we first need a Lemma.

\begin{lemma}
\label{lemma:single_bound}
If $Z$ is $s$-resilient from above, then $\forall \delta \in (0, 1)$,
$$
    \Pr\left[Z \leq \E[Z] - s(\delta)\frac{1-\delta}{\delta}\right] \leq \delta.
$$
\end{lemma}

For some $r \in \R$, let $B_r$ be the event $Z \leq \E[Z] - r$. 
%
\begin{align*}
    &\Pr\left[Z_{(m-k)} \leq \E[Z] - r\right] \\
    &= \Pr\left[ \sum_{j=0}^m \mathbb{I}(Z_j \not\in B_r) \leq k \right] \\
    &= \sum_{i=0}^k \binom{m}{i} \Pr[B_r]^{m-i} (1 - \Pr[B_r])^i \numberthis\label{eq:order_stats} \\
    &\leq \left( \sum_{i=0}^k \binom{m}{i}(1 - \Pr[B_r])^{i} \right) \Pr[B_r]^{m-k} \\
    &\leq \left( \sum_{i=0}^k (m(1 - \Pr[B_r]))^{i} \right) \\
    &\leq (m(1 - \Pr[B_r])+1)^k \Pr[B_r]^{m-k} 
\end{align*}

Now we will show that the algorithm chooses $v_i$, $i=0, \dotsc, T$ such that 
\begin{align*}
    (m(1 - v_i)+1)^k v_i^{m-k} \leq \zeta \numberthis\label{eq:rhs} 
\end{align*}

To show this we first show that $v_{i+1} \geq v_i$, $\forall i = 0, \dotsc, T-1$. This is obviously true for $i=0$ because $0 \leq v_0 \leq 1$ so 
\begin{align*}
   v_0 &= \stuffn \\
   &\leq \left(\frac{\zeta}{(m(1-v_0)+1)^k}\right)^{\frac{1}{m-k}}= v_1
\end{align*}
and we can proceed to use induction and conclude 
\begin{align*}
    v_i &= \left(\frac{\zeta}{(m(1-v_{i-1})+1)^k}\right)^{\frac{1}{m-k}} \\
    &\leq \left(\frac{\zeta}{(m(1-v_{i})+1)^k}\right)^{\frac{1}{m-k}} = v_{i+1}
\end{align*}

Now we use induction to show that Eq.(\ref{eq:rhs}) is true.
First at iteration 0 this is true because 
\begin{align*}
(m(1 - v_0)+1)^k v_0^{m-k} \leq (m+1)^k v_0^{m-k} 
\end{align*}
which is less than $\zeta$ as long as 
\begin{align*}
   v_0 \leq \stuffn 
\end{align*}

Suppose Eq.(\ref{eq:rhs}) is true at iteration $i$, then at iteration $i+1$ it is still true. Denote the values as $v_i$ and $v_{i+1}$. Because we choose $v_{i+1}$ such that 
\begin{align*}
    (m(1-v_i)+1)^k v^{m-k}_{i+1} = \zeta
\end{align*}
Observe that $v_{i+1} \geq v_i$ it must be true that
\begin{align*}
    (m(1-v_{i+1})+1)^k v^{m-k}_{i+1} \leq \zeta
\end{align*}

After we have established Eq.(\ref{eq:rhs}) we can apply Lemma~\ref{lemma:single_bound} to achieve $\Pr[B_r] \leq v$ it suffices to have
$$
    r \geq s\left( v \right) \left(v^{-1} - 1 \right)
$$

What remains is to prove Lemma~\ref{lemma:single_bound} 
\begin{proof}[Proof of Lemma~\ref{lemma:single_bound}]
Let $B_r$ be the event $Z \leq \E[Z] - r$ for any $r \geq 0$. Suppose it is true that $\Pr[B_r] = \delta_r$ for some $\delta_r \in (0, 1)$, we have
\begin{align*}
\E[Z] &= \E[Z\vert B_r]\Pr[B_r] + \E[Z \vert \bar{B}_r] (1 - \Pr[B_r]) \\
&\leq (\E[Z] - r)\delta_r + (\E[Z] + s(\delta_r))(1 - \delta_r)
\end{align*}
which implies that
$$
    r \leq s(\delta_r)\frac{1 - \delta_r}{\delta_r}
$$
which implies
$$
    \Pr\left[Z \leq \E[Z] - s(\delta)\frac{1-\delta}{\delta}\right] \leq \delta
$$
\end{proof}
\end{proof}

\asymptote*
\begin{proof}[Proof of Corollary~\ref{cor:asymptote}]
Denote $\epsilon = \stuffn$. We know that 
\begin{align*} 
\lim_{m \to \infty} \log \epsilon &= \frac{1}{m-k} \left( \log \zeta - k\log(m+1) \right) = 0
\end{align*}
which implies $\lim_{m \to \infty} \epsilon = 1$, so $\forall a < 0$ 
there must exist a $M$ such that $\forall m \geq M$, $(1 - \epsilon)^a + b \leq 2(1 - \epsilon)^a$, and $\epsilon > 1/2$
\begin{align*}
    r &\leq 2(1 - \epsilon)^a (\epsilon^{-1} - 1) = \frac{2(1 - \epsilon)^{1+a}}{\epsilon} \leq 4(1 - \epsilon)^{1+a} \\
\end{align*}
Observe that $\epsilon < 1$ so $1 - \epsilon < -\log \epsilon$ so we have (for sufficiently large $m$)
\begin{align*}
    r &\leq 4(-\log \epsilon)^{1+a} \\
    &= 4 \left( -\frac{1}{m-k} \log \zeta + \frac{k}{m-k} \log (m+1)\right)^{1+a} \\
    &\leq \frac{5}{m^{1+a}} \left( \log \frac{1}{\zeta} + k \log m \right)^{1+a} \\
    &\leq \frac{5}{m^{1+a}} \left( \log \frac{1}{\zeta} + k \log m \right)
\end{align*}
Now we only need the special case of $a = 0$, where 
\begin{align*} 
r \leq (1+b) (\epsilon^{-1}-1) \leq 2(1+b) (1 - \epsilon) 
\end{align*}
and the proof will follow as before. 
\end{proof}

\resilient*
\begin{proof}[Proof of Lemma~\ref{lemma:example_resilient}]
Part 1 is trivial to prove. Part 2 is proved in \citet{steinhardt2018robust}, Example 2.7. Part 3 is proven here.

Let $F$ denote the CDF of $Z$. For convenience, let $\bare = 1 - \epsilon$, $\tau = F^{-1}(\bare)$, and without loss of generality assume $\Eb[Z] = 0$. We first consider lower bounds on $\Eb[Z \mid Z \in B]$ where $B$ is any subset of $\Zc$ such that $\Pr[Z \in B] \geq \bare$. It is easy to see that for any such $B$ we have
\[ \Eb[Z \mid Z \leq \tau] \leq \Eb[Z \mid Z \in B] \]
so we only have to provide a lower bound for $\Eb[Z \mid Z \leq \tau]$. Without loss of generality we can also assume $\tau \leq 0$ because suppose $\tau > 0$ then consider an alternative random variable $\tilde{Z}$ defined by $\tilde{Z} = \max(Z, 0)$. Then $\tilde{Z}$ is $\sigma^2$ sub-Gaussian, and 
\[ \Eb[Z \mid Z \leq \tau] \geq \Eb[\tilde{Z} \mid \tilde{Z} \leq \tau] \]
Then we can provide a lower bound for $\Eb[\tilde{Z} \mid \tilde{Z} \leq \tau]$ instead. Given the above setup we have
\begin{align*}
    & \bare\, \E[Z \vert Z \leq \tau] \\
    &= \int_{x=-\infty}^{\tau} x F'(x)\, dx
     = \int_{x=-\infty}^{\tau} F'(x) \int_{y = x}^0 (-1)\, dy\, dx \\
    &= -\int_{x=-\infty}^{\tau} \int_{y = x}^0 F'(x)\, dy\, dx \\
    &= -\int_{x=-\infty}^{\tau} \int_{y = -\infty}^0 \mathbb{I}(y > x) F'(x)\, dy\, dx \\
    &= -\int_{y=-\infty}^0 \int_{x=-\infty}^{\tau} \mathbb{I}(y > x) F'(x)\, dx\, dy \\
    &= -\int_{y=\tau}^0 \int_{x=-\infty}^{\tau} \mathbb{I}(y > x) F'(x)\, dx \, dy  \\
    &\qquad - \int_{y=-\infty}^{\tau} \int_{x=-\infty}^{\tau} \mathbb{I}(y > x) F'(x)\, dx\, dy \\
    &= -\int_{y = \tau}^0 \bare\, dy
     - \int_{y = -\infty}^{\tau} \int_{x=-\infty}^{y} F'(x)\, dx\, dy \\
    &= \bare \tau - \int_{y=-\infty}^{\tau} F(y)\, dy \\
    &= \bare \tau - \int_{x=-\infty}^{\tau} F(x)\, dx
\end{align*}

Let $ \tF(x) = e^{-\frac{x^2}{2\sigma^2}} $. Since $Z$ is $\sigma^2$ sub-Gaussian, by Chernoff bound we know that $\forall x < 0.\ F(x) \leq \tF(x)$, which also implies that whenever $F^{-1}(\bare) < 0$, $\tF^{-1}(\bare) \leq F^{-1}(\bare)$, Then 
\begin{align*}
    &\bare \E[Z \vert Z \leq F^{-1}(\bare)] \\
    &= \bare F^{-1}(\bare) - \int_{x = -\infty}^{F^{-1}(\bare)} F(x) dx \\
    &= \bare F^{-1}(\bare) - \int_{x = -\infty}^{\tF^{-1}(\bare)} F(x) dx - \int_{x = \tF^{-1}(\bare)}^{F^{-1}(\bare)} F(x) dx  \\
    &\geq \bare F^{-1}(\bare) - \int_{x = -\infty}^{\tF^{-1}(\bare)} \tF(x) dx - \int_{x = \tF^{-1}(\bare)}^{F^{-1}(\bare)} \bare dx \\
    &= \bare \tF^{-1}(\bare) - \int_{x = -\infty}^{\tF^{-1}(\bare)} \tF(x) dx
\end{align*}
Finally, denote $\phi_\sigma(x)$ as the PDF of $\Nc(0, \sigma^2)$ and $\Phi_\sigma$ be its CDF, we have 
\begin{align*}
    \int_{x = -\infty}^{\tF^{-1}(\bare)} \tF(x) dx &= \sqrt{2\pi}\sigma \int_{x = -\infty}^{\tF^{-1}(\bare)} \phi_\sigma(x) dx \\
    &= \sqrt{2\pi}\sigma \Phi_\sigma(\tF^{-1}(\bare))  \leq \sqrt{2\pi}\sigma\bare
\end{align*}
Combining these results we get 
\begin{align*}
    &\E[Z \vert Z \leq F^{-1}(\bare)] \\
    &\geq \tF^{-1}(\bare) - \sqrt{2\pi}\sigma = -\sqrt{2\sigma\log\frac{1}{\bare}} - \sqrt{2\pi}\sigma
\end{align*}

\end{proof}


\boundedexample*
\begin{proof}[Proof of Corollary~\ref{cor:bounded_example}]
By Theorem~\ref{thm:resilient} we have
\begin{align*}
    \zeta \geq &\Pr\left(Z_{(1+k)} \geq \E[Z] + (\E[Z] - a) (v_T^{-1} - 1) \right)\\
    & = \Pr\left(Z_{(1+k)} \geq v^{-1}_T \E[Z] - v^{-1}_T a + a \right) \\
    &= \Pr\left(v_T Z_{(1+k)} \geq v_T a -  a + \E[Z]  \right) \\
    &= \Pr\left( \E[Z] \leq a + v_T(Z_{(1+k)} - a) \right)
\end{align*}
Similarly we can conclude
\begin{align*}
    \zeta \geq \Pr\left( \E[Z] \geq b - v_T(b - Z_{(m-k)}) \right)
\end{align*}
\end{proof}

\multidim*
\begin{proof}[Proof of Theorem~\ref{thm:multi-dim}]
Let 
\[
    v^* = \argsup_{v, \vdual=1} \langle v, \E[Z] \rangle
\]
then $\lVert \E[Z] \rVert = \langle v^*, \E[Z] \rangle$. Let $\tZ_{(1)}, \dotsc, \tZ_{(m)}$ be ranked such that
\[
    \langle v^*, \tZ_{(1)} \rangle \leq \dotsc \leq \langle v^*, Z_{(m)} \rangle.
\]
Denote the event $B \subset \Zc$ as $\{ Z, \langle v^*, Z \rangle \leq \langle v^*, \E[Z] \rangle - r\}$ as before we have
\begin{align*}
    \Pr\left[ \langle v^*, \tZ_{(m-k)} \rangle \leq \langle v^*, \E[Z] \rangle - r \right] \leq  (m+1)^k \Pr[B]^{m-k}
\end{align*}
and we set the RHS $\leq \zeta$. It suffices to have 
\[ \Pr[B] \leq \stuffn \]
Similar to Lemma~\ref{lemma:single_bound}, denote $\delta = \Pr[B]$
\begin{align*}
    &\langle v^*, \E[Z] \rangle = \E[ \langle v^*, Z \rangle] \\
    &= \E[ \langle v^*, Z \rangle \mid B]\Pr[B] + \E[ \langle v^*, Z \rangle \mid \bar{B}](1 - \Pr[B]) \\
    &\leq (\langle v^*, \E[Z] \rangle - r) \delta +  (\langle v^*, \E[Z] \rangle + s(\delta))(1 - \delta)
\end{align*}
which implies that $\Pr[B] \leq \delta$ when $r \geq s(\delta)\frac{1-\delta}{\delta} = s(\delta)(\delta^{-1} - 1)$. When this is true we have
\begin{align*}
    \Pr[\langle v^*, \tZ_{(m-k)} \rangle \leq \lVert \E[Z] \rVert - r] \leq \zeta
\end{align*}
If we also rank $Z_{(1)}, \dotsc, Z_{(m)}$ by
\begin{align*}
    \lVert Z_{(1)} \rVert \leq \cdots \leq \lVert Z_{(m)} \rVert
\end{align*}
we have $\forall i \geq m-k$
\begin{align*}
    \langle v^*, \tZ_{(m-k)} \rangle \leq \langle v^*, \tZ_{(i)} \rangle \leq \lVert \tilde{Z}_{(i)} \rVert
\end{align*}
so there are at least $k$ samples $Z_{(i)}$ with norm at least $\langle v^*, \tZ_{(m-k)} \rangle$, and we can conclude that $ \langle v^*, \tZ_{(m-k)} \rangle \leq \lVert Z_{(m-k)} \rVert$ which implies 
\begin{align*}
     \Pr[\lVert Z_{(m-k)} \rVert \leq \lVert \E[Z] \rVert - r] \leq \zeta
\end{align*}
\end{proof}

\end{document}